\documentclass[letterpaper]{article} 
\usepackage{aaai25}  
\usepackage{times}  
\usepackage{helvet}  
\usepackage{courier}  
\usepackage[hyphens]{url}  
\usepackage{graphicx} 
\usepackage{natbib}  
\usepackage{caption} 
\usepackage{algorithm}
\usepackage{algorithmic}
\usepackage{amsmath} 
\usepackage{multirow}
\usepackage{amssymb}
\usepackage{booktabs}
\usepackage{xcolor}

\usepackage{amsthm}
\newtheorem{theorem}{Theorem}

\newtheorem{proposition}{Proposition}
\newtheorem{assumption}{Assumption}
\usepackage{makecell}

\usepackage{newfloat}
\usepackage{listings}
\DeclareCaptionStyle{ruled}{labelfont=normalfont,labelsep=colon,strut=off} 
\floatstyle{ruled}
\newfloat{listing}{tb}{lst}{}
\floatname{listing}{Listing}
\title{No More Tuning: Prioritized Multi-Task Learning with Lagrangian Differential Multiplier Methods}

\author{
    Zhengxing Cheng\textsuperscript{\rm 1,2}\thanks{This work was completed during an internship at Alibaba Group.}, 
    Yuheng Huang\textsuperscript{\rm 2},
    Zhixuan Zhang\textsuperscript{\rm 2},
    Dan Ou\textsuperscript{\rm 2},
    Qingwen Liu\textsuperscript{\rm 2}
}
\affiliations{
    \textsuperscript{\rm 1}University of California, Berkeley\\
    \textsuperscript{\rm 2}Alibaba Group\\

    zhengxing\_cheng@berkeley.edu,
    \{hexing.hyh,zhibing.zzx,oudan.od,xiangsheng.lqw\}@alibaba-inc.com
}

\usepackage{bibentry}

\begin{document}

\maketitle

\begin{abstract}
Given the ubiquity of multi-task in practical systems, Multi-Task Learning (MTL) has found widespread application across diverse domains. In real-world scenarios, these tasks often have different priorities. For instance, In web search, relevance is often prioritized over other metrics, such as click-through rates or user engagement. Existing frameworks pay insufficient attention to the prioritization among different tasks, which typically adjust task-specific loss function weights to differentiate task priorities. However, this approach encounters challenges as the number of tasks grows, leading to exponential increases in hyper-parameter tuning complexity. Furthermore, the simultaneous optimization of multiple objectives can negatively impact the performance of high-priority tasks due to interference from lower-priority tasks.

In this paper, we introduce a novel multi-task learning framework employing Lagrangian Differential Multiplier Methods for step-wise multi-task optimization. It is designed to boost the performance of high-priority tasks without interference from other tasks. Its primary advantage lies in its ability to automatically optimize multiple objectives without requiring balancing hyper-parameters for different tasks, thereby eliminating the need for manual tuning. Additionally, we provide theoretical analysis demonstrating that our method ensures optimization guarantees, enhancing the reliability of the process. We demonstrate its effectiveness through experiments on multiple public datasets and its application in Taobao search, a large-scale industrial search ranking system, resulting in significant improvements across various business metrics.

\end{abstract}

%

\section{Introduction}
In the era of data-driven decision-making, Multi-Task Learning (MTL) has become an essential paradigm, offering notable advantages in managing multiple objectives concurrently~\citep{mtl_overview_nlp,mtl_overview_recommandation}. This is especially relevant in practical systems where multiple objectives must be optimized simultaneously, even though these objectives may sometimes negatively influence each other.

A illustrative example is in the domain of web search, where relevance is typically prioritized over other metrics such as click-through rates (CTR) or user engagement. In a web search system, relevance refers to how well the returned results match the user's query, directly impacting user satisfaction and search quality. While metrics like CTR and engagement are also important, prioritizing relevance ensures that users find the most useful and accurate information first. However, optimizing for secondary objectives like CTR can negatively affect the primary goal of relevance. For example, to increase CTR, a system might prioritize results that are more attention-grabbing but less relevant to the query, thereby reducing the overall relevance of the search results. This trade-off underscores the need for a framework that can effectively balance these tasks without sacrificing the primary goal of relevance.

Despite its widespread adoption, traditional MTL frameworks often encounter challenges in task prioritization, which is typically addressed by adjusting task-specific loss function weights. However, the complexity of hyper-parameter tuning escalates exponentially as the number of tasks increases, especially when using grid search methods. In real-world industrial applications, this complexity is exacerbated by the necessity for online A/B testing to validate the final model effectiveness. A/B testing requires sustained online observation to gather feedback, rendering this exponential complexity impractical in industrial scenarios due to the extended duration and high resource consumption.

To tackle these challenges, we introduce a novel multi-objective optimization framework called No More Tuning (NMT). This framework manages task prioritization in multi-task learning (MTL) by ensuring that secondary tasks are optimized without compromising the performance of the primary task. This is achieved by framing the MTL problem as a constrained optimization problem, where the primary task's performance is maintained as a inequality constraint during the optimization of secondary tasks.

To solve this constrained problem, we use the method of Lagrange multipliers, which allows us to convert it into an unconstrained problem.To accommodate the gradient descent optimization methods widely used in most MTL approaches, we discuss solving the Lagrangian function using gradient descent. Consequently, the NMT framework can seamlessly integrate with any MTL method that uses gradient descent optimization algorithms. As a result, NMT can be integrated with any gradient descent-based MTL method. Notably, since the task prioritization is incorporated into the inequality constraints, no additional hyper-parameters are required in NMT framework. 

\begin{figure}[t]
    \centering
    \includegraphics[width=0.5\textwidth]{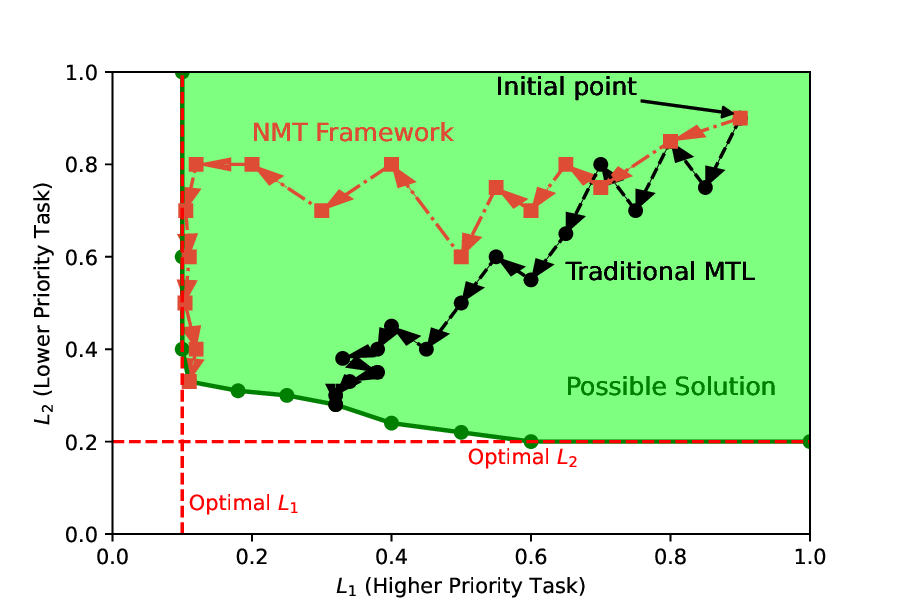}
    \caption{Optimization trajectories for two strategies. The traditional approach compromises between tasks, often sub-optimally affecting the primary objective. The NMT framework first optimizes the primary task, then refines secondary tasks while maintaining the primary task's performance.}
    \label{fig:trajectory_demo}
\end{figure}

Figure \ref{fig:trajectory_demo} gives a visual representation of the optimization process with the NMT framework. The figure shows the trajectories of two different optimization strategies, highlighting how the NMT framework effectively balances multiple objectives. The traditional approach attempts to find a compromise between tasks, often resulting in sub-optimal performance for the primary objective. In contrast, the trajectory for the NMT framework demonstrates a distinct path through the optimization space. Initially, this path focuses solely on optimizing the primary task. Once the primary task achieves its optimal performance, the framework then optimizes secondary tasks while maintaining the performance of the primary task. This trajectory visually emphasizes how the NMT framework ensures that the primary objective remains uncompromised while improving secondary tasks.

The key advantages of our method include:

\textbf{No Need for Parameter Adjustments:} The NMT framework embeds task prioritization directly into the equality constraints of the optimization problem. This eliminates the need for manual parameter tuning or adjustments, which are often necessary in traditional MTL methods to balance different tasks. By removing this requirement, our framework significantly reduces the complexity and time involvxd in the model training process. This also reduces the risk of sub-optimal performance due to improper parameter settings.

\textbf{Theoretical Background:}  A key advantage of the NMT framework is its ability to provide theoretical background regarding the performance of high-priority tasks. We give the theoretical proof in section Theoretical Analysis. By setting the primary task's performance as a strict inequality constraint, the framework ensures that the optimization direction adheres to these constraints. This contrasts with traditional methods that might require extensive empirical validation to ensure that primary tasks are not adversely affected.

\textbf{Easy Integration:} The NMT framework can be easily integrated with all gradient descent-based MTL methods, offering a simple yet effective solution for task prioritization. 

\textbf{Broad Effectiveness:} We conducted experiments on multiple multi-task public datasets, to validate the effectiveness of our method. Furthermore, the NMT framework has been implemented in Taobao search, a large-scale industrial online systems with billions of active users, resulting in significant gains in online A/B testing. This demonstrates the framework's versatility and effectiveness across different domains and applications.

\section{Related Works}
\subsection{Overview of MTL}
Multi-Task Learning (MTL) has become increasingly prevalent in real world applications, particularly in recommendation and search engines where multiple objectives should be considered. Traditional methods that optimizing separate task-specific loss functions can result in conflicts and reduced performance across tasks \citep{mtl:Sharedbottom}. To overcome these challenges, various research has been conducted to balance the trade-offs between tasks.

In Multi-task learning, one active research area is parameter-sharing architecture. We provide an overview of the evolution of MTL architectures, which can be categorized into four principal types: Tower-level task-specific models \cite{mtl:Sharedbottom, mtl:crosshare_bottom}, Gate-level task-specific models\cite{gate-level}, Expert-level task-specific models\cite{mtl:PLE} and Embedding-level task-specific models \cite{mtl:STEM}. The parameter-sharing framework enables multiple tasks to utilize a common set of experts, with task-specific gating networks directing the flow of information. 

Various MTL optimization methods have been explored to address negative transfer between tasks. A prevalent approach is task weighting, which balances multiple tasks by combining loss objectives with specific weights, typically determined through extensive hyper-parameter tuning \cite{priority:5, priority:6}. Dynamic gradient-based methods \cite{priority:4, priority:7, priority:8} have been developed to manage task weight balancing by employing techniques such as gradient normalization and adjustment to reduce disparities between tasks. Pareto optimization framework\cite{priority:9} has also been used for coordinate these objectives with a weighted aggregation. However, these methods can not specifically prioritize one task over others in the optimization process.

Recent studies have explored more principled approaches to MTL, such as the framework developed by \cite{mtl:3}, which integrates first-order gradient-based algorithms into the ranking process. Further advancements in MTL, like those introduced by \cite{mtl:1}, involve training separate models on distinct objectives and then aggregating their scores with stochastic label aggregation. Additionally, distillation-based ranking solutions \cite{mtl:2,mtl:4} use the distilled models to balance the optimization of multiple tasks in the learning-to-rank setting. However, these approaches all still rely on parameter tuning to balance task performance and fail to explicitly guarantee the prioritization of critical tasks. When implementating in the real-world system, the practical challenges including tuning and computational complexity persist.

In contrast to these techniques, our optimization framework offers a novel approach that simplifies the optimization process and enhances the performance of high-priority tasks without requiring extensive manual tuning. It is straightforward to implement and applicable across various MTL architectures.

\subsection{Constrained Optimization and Application}

Constrained optimization is a classical optimization problem, with Lagrangian methods being among the most well-established approaches \cite{co:original_paper}. Constrained optimization has been widely applied in reinforcement learning (RL), particularly in safe RL tasks \cite{co:PID}. By using Lagrangian methods, the original constrained problem is transformed into an unconstrained form, facilitating optimization.  Theoretical convergence proofs and methods for updating Lagrange multipliers have been developed to support this process\cite{co:0dualitygap, co_Reward_Constrained_Optimization}.

Inspired by the successful application of constrained optimization as well as the Lagrangian methods in RL, our work aims to adapt these techniques to Multi-Task Learning (MTL) problems. We propose reformulating the prioritized multi-objective optimization problem as a multi-step constrained optimization problem, utilizing the principles of constrained optimization to tackle challenges in MTL.

\section{Problem Formulation}
Without loss of generality, consider a scenario with two tasks, where Task 1 has higher priority than Task 2. Task 1 and Task 2 are represented by loss functions 
$f_1(\theta)$ and $f_2(\theta)$, respectively. The optimization problem can be formulated as:

\begin{equation}
    \min_{\theta} f_1(\theta), f_2(\theta)
\end{equation}

Our objective is as follows:
\begin{itemize}
    \item Primary Task (Higher Priority): Minimize $f_1(\theta)$
    \item Secondary Task (Lower Priority): Minimize $f_2(\theta)$ while ensuring that the performance of $f_1(\theta)$ is not compromised.
\end{itemize}

The question can be extended to any number of tasks. Consider a scenario with \( n \) tasks, where each task \( i \) has an associated loss function \( f_i(\theta) \) and a priority level \( p_i \). The priority levels are such that \( p_1 \) is the highest priority and \( p_n \) is the lowest, with \( p_i \) being monotonic (i.e., \( p_1 > p_2 > \cdots > p_n \)). The optimization problem can be formulated as:

\begin{equation}
    \min_{\theta} f_1(\theta), f_2(\theta), \ldots, f_n(\theta)
\end{equation}

The requirement is to optimize each subsequent tasks \( f_i(\theta) \) for \( i > 1 \) while ensuring that the performance of higher priority tasks \( f_j(\theta) \) for \( j < i \) is maintained.

\section{Proposed No More Tuning (NMT) Optimization Framework}
\subsection{Optimization under two tasks}
Given that \( f_1(\theta) \) is prioritized, we need \( f_2(\theta) \) to be optimized while keeping \( f_1(\theta) \) at its optimal value \( f_1(\theta^*) \), found from:

\begin{equation}
    \theta^* = \arg \min_{\theta} f_1(\theta)
\end{equation}

The optimization problem for \( f_2(\theta) \) is then:

\begin{equation}
\setlength\arraycolsep{1.5pt}
\begin{array}{l@{\qquad}  c r c r}
    \displaystyle  \min_{\theta}   &   f_2(\theta) \\
         \text{s.t.} &  f_1(\theta) \leq f_1(\theta^*)  \\
\end{array}
\label{eq:constrain_two_task}
\end{equation}

Solving this problem ensures that the performance of the high-priority task remains unaffected while optimizing the low-priority task. The NMT framework aims to address this optimization problem by employing gradient descent.

\subsection{Extension to arbitrary $m$ tasks}
The above framework can naturally be extended to any number of tasks, represented by loss functions \( f_1(\theta), f_2(\theta), \ldots, f_m(\theta) \), with Task 1 having the highest priority and Task $m$ the lowest. The problem can be formulated as:

\begin{equation}
    \min_{\theta} f_1(\theta), f_2(\theta), \ldots, f_m(\theta)
\end{equation}

The optimization steps are:

\begin{enumerate}
    \item Minimize \( f_1(\theta) \) (Highest Priority) to obtain \( \theta^*_1 \) and \( f_1(\theta^*_1) \).
    \item Minimize \( f_2(\theta) \) (Secondary Priority) subject to \( f_1(\theta) \leq f_1(\theta^*_1) \), to obtain \( \theta^*_2 \).
    \item Minimize \( f_3(\theta) \) (Tertiary Priority) subject to \( f_1(\theta) \leq f_1(\theta^*_2) \) and \( f_2(\theta) \leq f_2(\theta^*_2) \), to obtain \( \theta^*_3 \). 
    \item Continue this process for subsequent tasks.
\end{enumerate}

For each task \( f_i(\theta) \), the problem is formulated as:

\begin{equation}
\label{eq_m}
\begin{array}{l@{\qquad}  c r c r}

   \displaystyle   \min_{\theta}   &   f_i(\theta) \\
         \text{s.t.} &  f_1(\theta) \leq f_1(\theta^*_{i-1}) \\
                      &  f_2(\theta) \leq f_2(\theta^*_{i-1}) \\
                      &  \vdots \\
                      &  f_{i-1}(\theta) \leq f_{i-1}(\theta^*_{i-1})\\
\end{array}
\end{equation}

This approach ensures that each task \( f_i(\theta) \) is optimized in sequence until the $m_{th}$ task while preserving the optimal values of all higher-priority tasks.

\subsection{No More Tuning (NMT) Algorithm}
The overall No More Tuning (NMT) optimization framework for $m$ tasks is illustrated in Algorithm \ref{alg:algorithm}. We convert the inequality constraint problem (\ref{eq_m}) into an unconstrained optimization problem using the Lagrange multiplier method and solve it with a gradient-based approach.

\begin{algorithm}[h]
\caption{NMT Algorithm for $m$ Tasks}
\label{alg:algorithm}
\textbf{Input}: 
\begin{itemize}
    \item $\eta$: Learning rate for model parameters
    \item $\tau$: Learning rate for Lagrange multipliers
    \item $f_k(\theta)$: Objective function for task $k$ ($k \in \{1, 2, \dots, m\}$)
    \item $\lambda_{\text{init}}$: Initial value for Lagrange multipliers
\end{itemize}

\begin{algorithmic}[1]
\STATE \textbf{Step 1:} Optimize the primary task (task 1) to minimize $f_1(\theta)$ until convergence. Save the optimized parameters as $\theta^*$.
\FOR{each task $k$ from 2 to $m$}
    \STATE Initialize $\theta$ with $\theta^*$ and $\lambda_j$ with $\lambda_{\text{init}}$ for $j = 1, \ldots, k-1$.
    \REPEAT
        \STATE Compute the aggregate loss for task $k$:
        \[
        \mathcal{L} = f_k(\theta) + \sum_{j=1}^{k-1} \lambda_j \cdot (f_j(\theta) - f_j(\theta^*))
        \]
        \STATE Update $\theta$ using gradient descent:
        \[
        \theta \gets \theta - \eta \cdot \nabla_{\theta} \mathcal{L}
        \]
        \STATE Update each $\lambda_j$ using gradient ascent:
        \[
        \lambda_j \gets \lambda_j + \tau \cdot (f_j(\theta) - f_j(\theta^*))
        \]
    \UNTIL{convergence}
    \STATE Update $\theta^*$ with the optimized $\theta$
\ENDFOR
\end{algorithmic}
\end{algorithm}

The algorithm is divided into two main stages:
\begin{enumerate}
    \item \textbf{Optimize the main task}: First, we perform optimization for the first task (the main task) to find the parameters $\theta^*$ that minimize the objective function $f_1(\theta)$. This stage lasts until the objective function converges.
    \item \textbf{Iteratively optimize the remaining tasks}: After the optimization of the main task, we optimize each of the remaining tasks in turn. For each task $k$ (starting from the second task), we initialize the model parameters to $\theta^*$ and set the Lagrange multiplier to the initial value $\lambda_{\text{init}}$. Then, the model parameters are updated using the gradient descent algorithm, while the Lagrange multiplier is updated using the gradient ascent algorithm. This process continues until convergence. After each iteration, we update the current optimized parameters to $\theta^*$ in preparation for the optimization of the next task.
\end{enumerate}

The NMT framework is very easy to implement because it only requires gradient ascent of the Lagrange multiplier $\lambda$ based on the optimization of the secondary task. No additional hyper-parameters are introduced in the whole framework.

\subsubsection{Re-scaling method}
When implementing the algorithm, we find that when $\lambda$ is large, the loss function can become excessively large, leading to substantial updates in the parameters $\theta$ during gradient descent. These large updates can destabilize the learning process and violate the assumption outlined in the next chapter, where we aim for small changes in $\theta$. Thus, we apply a re-scaled loss function when updating $\theta$. When performing gradient descent on $\theta$,
$$
\mathcal{L} = \frac{1}{1 + \sum_{j=1}^{k-1} \lambda_j} \left( f_k(\theta) + \sum_{j=1}^{k-1} \lambda_j \left( f_j(\theta) - f_j(\theta^\star_j) \right) \right)
$$

This re-scaling formulation ensures the combination of each loss is a normalized convex combination, preventing any potential explosion of the loss function. All our experiments incorporate this re-scaling formulation for the loss function during the optimization process.

\section{Theoretical Analysis}
\label{sec:theory_analysis}






We now consider a relaxed version of the problem that allows for some tolerance in the constraints. Specifically, we examine the constrained optimization problem (CO) with $m$ objectives, where the $i_{th}$ objective is allowed to perform at most $r_i$ worse than $f_i(\theta^*)$:

\begin{equation}
\setlength\arraycolsep{1.5pt}
\begin{array}{l@{\qquad}  c r c r}
     \displaystyle \min_{\theta}   &   f_m(\theta) \\
         \text{s.t.} &  f_i(\theta)) \leq f_i(\theta^*)) + r_i, i = 1,... m-1,  \\
\end{array}
\label{eq:constrain}
\tag{CO}
\end{equation}
where $f_i$ is the $i_{th}$ objective function and $r_i$ is a small positive tolerance term for the $i_{th}$ objective.

From the constrained problem, we can define the Lagrangian function as:
\begin{equation}
L(\theta, \lambda) = f_m(\theta) +  \sum_{i=1}^{m-1} \lambda_i \left( f_i(\theta) - f_i(\theta^*) - r_i \right).
\end{equation}

Then the unconstrained dual optimization problem (DO) can then be formulated as:
\begin{equation}
\max_\lambda\min_\theta L(\theta, \lambda)
\tag{DO}
\end{equation}

If $f_i$ is a convex function (e.g. in logistic regression), the strong duality between CO and DO holds \cite{co:original_paper} and the convergence of the NMT optimization is guaranteed .

However, in many practical scenarios, the objective functions are not convex. We aim to show that, under appropriate assumptions, strong duality still holds despite non-convexity, guiding our optimization method.

\begin{assumption}
The model training is free from over-fitting.
\end{assumption}
The first assumption ensures that the optimal solution corresponds to the minimum of the objective function. It guarantees monotonicity between the objective function and the target performance.

\begin{assumption}
Within the feasible region satisfying the constraints, the difference of $\theta$ is bounded by $\epsilon$, and $f_i$ is Lipschitz continuous.
\end{assumption}

This assumption specifies properties of the objective function useful for analyzing duality.

We also introduce the definition of the perturbation function associated with CO for later proofs.

\begin{equation}
\begin{aligned}
P(\xi) \triangleq & \min _{\theta} f_m(\theta) \\
& \operatorname{subject} \text { to } f_i(\theta) \leq f_i(\theta^*) + r_i - \xi_i, i=1 \ldots m-1
\end{aligned}
\label{PF}
\end{equation}

We now demonstrate two properties under the given assumptions:

\begin{proposition}
Slater's condition holds for CO
\end{proposition}
\begin{proof}
As $\theta^*$ is a feasible solution and $r_i$ is positive, $ f_i(\theta^*) \leq f_i(\theta^*) + r_i $ satisfies trivially.
\end{proof}

\begin{proposition}
Under assumption 2, perturbation function $P(\xi)$ is approximately convex when $\epsilon$ is small enough.
\end{proposition}

The detailed proof of Proposition 2 will be provided in the Appendix of the extended version, which has been published on arXiv. Denote $\theta^*(\xi)$ as optimal $\theta$ for given specific $\xi$. We derive that when $\|\theta^*(\xi_2) - \theta^*(\xi_1)\|$ which is bounded by the $\epsilon$ is sufficiently close to 0, the convexity holds. 

\begin{theorem}
Under Assumption 1,2, when $\epsilon$ is small enough, the strong duality holds for CO and DO.
\end{theorem}
\begin{proof} 
The strong duality holds when two conditions are satisfied Slater's condition and the convexity of the perturbation function\cite{proof:1}. We have demonstrated both conditions in the preceding propositions.
\end{proof} 

Theorem 1 indicates that if the change in $\theta^*$ is sufficiently small, the optimal solution of DO will also be the optimal solution of CO. Then we can get a feasible solution by optimizing the unconstrained dual problem DO which is more convenient. The NMT framework aims to find the local minimum $\theta$ for $f_m$ that satisfies the constraints, ensuring minimal deviation from previously identified optimal $\theta^*$. The min-max optimization requirement suggests a two-step approach: gradient descent is performed on \(\theta\) to minimize \(f_i\), while gradient ascent is applied to \(\lambda\) to satisfy the maximization requirement. Consequently, our NMT optimization framework, which performs gradient descent on \(\theta\) and gradient ascent on \(\lambda\), is expected to converge to an optimal solution.

\section{Experiments}

\subsection{Experimental Results on Public Datasets}

\subsubsection{Experimental Setup}

In our experiments, we selected two public MTL recommendation datasets, namely TikTok and QK-Video \cite{exp:QKV} for performance evaluation. The TikTok dataset includes two objectives: Finish and Like, while the QK-Video dataset contains two objectives: Click and Like. In our experiments, we prioritize the Like objective as the primary task, with Finish and Click serving as the respective secondary tasks in the two datasets.

NMT is designed to be fully compatible with most existing MTL approaches. Rather than directly competing with these methods, it serves as a complementary framework that enhances their performance by emphasizing task prioritization. In this section, we integrate the NMT framework with parameter-sharing MTL architectures, including Shared-Bottom \cite{mtl:Sharedbottom}, OMoE \cite{mtl:OMoE}, MMoE \cite{mtl:OMoE}, and PLE \cite{mtl:PLE}. Additionally, we apply the NMT framework to a gradient-based MTL method \cite{famo}, with implementation details provided in the appendix. 

\subsubsection{Performance Evaluation}
\begin{table*}[h]
\centering
\caption{Overall performance on TikTok.}
\resizebox{1.0\textwidth}{!}{

    \begin{tabular}{lcccccc}
    \toprule
    \multirow{2}{*}{Model} & \multicolumn{3}{c}{Without NMT} & \multicolumn{3}{c}{With NMT} \\
    \cmidrule(lr){2-4} \cmidrule(lr){5-7}
     & Finish AUC & Like AUC & Average AUC & Finish AUC & Like AUC & Average AUC \\
    \midrule
    Shared-Bottom & 0.7504 & 0.9031 & 0.8267 & 0.7510 (+0.06\%) & 0.9069 (\textbf{+0.38\%}) & 0.8289 (+0.22\%) \\
    OMoE          & 0.7505 & 0.9021 & 0.8263 & 0.7509 (+0.04\%) & 0.9059 (\textbf{+0.38\%}) & 0.8284 (+0.21\%) \\
    MMoE          & 0.7503 & 0.9015 & 0.8259 & 0.7506 (+0.03\%) & 0.9056 (\textbf{+0.41\%}) & 0.8281 (+0.22\%) \\
    PLE           & 0.7506 & 0.9027 & 0.8266 & 0.7511 (+0.05\%) & 0.9076 (\textbf{+0.49\%}) & 0.8293 (+0.27\%) \\
    \bottomrule
    \end{tabular}
    }
    \label{table1}
\end{table*}

\begin{table*}[h]
\centering
\caption{Overall performance on QK-Video}
\resizebox{1.0\textwidth}{!}{

    \begin{tabular}{lcccccc}
    \toprule
    \multirow{2}{*}{Model} & \multicolumn{3}{c}{Without NMT} & \multicolumn{3}{c}{With NMT } \\
    \cmidrule(lr){2-4} \cmidrule(lr){5-7}
     & Click AUC & Like AUC & Average AUC & Click AUC & Like AUC & Average AUC \\
    \midrule
    Shared-Bottom & 0.9128 & 0.9409 & 0.9268 & 0.9126 (-0.02\%) & 0.9417 (\textbf{+0.07\%}) & 0.9271 (+0.03\%) \\
    OMoE          & 0.9125 & 0.9414 & 0.9270 & 0.9119 (-0.07\%) & 0.9424 (\textbf{+0.1\%}) & 0.9272 (+0.02\%) \\
    MMoE          & 0.9126 & 0.9412 & 0.9269 & 0.9125 (-0.01\%) & 0.9417 (\textbf{+0.05\%}) & 0.9271 (+0.02\%) \\
    PLE           & 0.9126 & 0.9422 & 0.9274 & 0.9122 (-0.04\%) & 0.9425 (\textbf{+0.03\%}) & 0.9273 (-0.01\%) \\
    \bottomrule
    \end{tabular}
}
\label{table2}
\end{table*}

As shown in Table 1, NMT led to significant improvements in Like AUC across all models in the TikTok dataset. This demonstrates that NMT effectively enhances the performance of the high-priority task. Meanwhile, the other metrics, such as Finish AUC and Average AUC, also saw modest improvements, indicating that the performance of these secondary tasks remains stable and does not deteriorate when prioritizing the Like task. The PLE model, particularly when paired with NMT, stands out with the highest AUC across all metrics.

Table 2 shows the performance on the QK-Video dataset. Similar to the TikTok dataset, Like AUC increased across all models with NMT While there were slight decreases in Click AUC for some models, these reductions were minimal, indicating that the prioritization of the Like task does not severely impact the performance of other metrics. The Average AUC remained largely stable, demonstrating that the overall model performance is well-balanced and not significantly compromised by the prioritization of the high-priority Like task.

Overall, these results demonstrate that NMT effectively prioritizes the high-priority task, improving its performance while maintaining satisfactory levels of performance in other metrics. In some conditions, it can even lead to simultaneous improvements for all the tasks, demonstrating the possibility of the NMT in enhancing overall task performance without compromising on any specific objective.

We also compares NMT with traditional weights adjustment methods. The weights adjustment methods directly sums the loss functions of the two tasks and manually adjusts the coefficients of the loss functions respectively:

\begin{equation}
    L = \alpha_1 L_1 + \alpha_2 L_2
\end{equation}

where $\alpha_1$ and $\alpha_2$ are tuned manually.

In order to fully demonstrate all the possibilities of the weights adjustment method, we adjusted as many hyperparameter combinations as possible. Figure \ref{fig:diffweight} show the AUC performance on the TikTok dataset with different combination of loss weights (from 0.1:0.9 to 0.5:0.5). The dotted lines represent the trade-offs between the two metrics as the weights shift. The star markers indicate the performance of models optimized by NMT. Notably, these NMT-optimized models significantly outperform their standard counterparts, achieving results far beyond the capabilities of the traditional loss weight adjustments. The NMT optimization framework consistently pushing the boundaries of individual tasks' performance which highlight the effectiveness of the NMT framework in optimizing both high-priority and secondary tasks simultaneously.

\begin{figure}[h]
   
    \centering
    \includegraphics[width=0.38\textwidth]{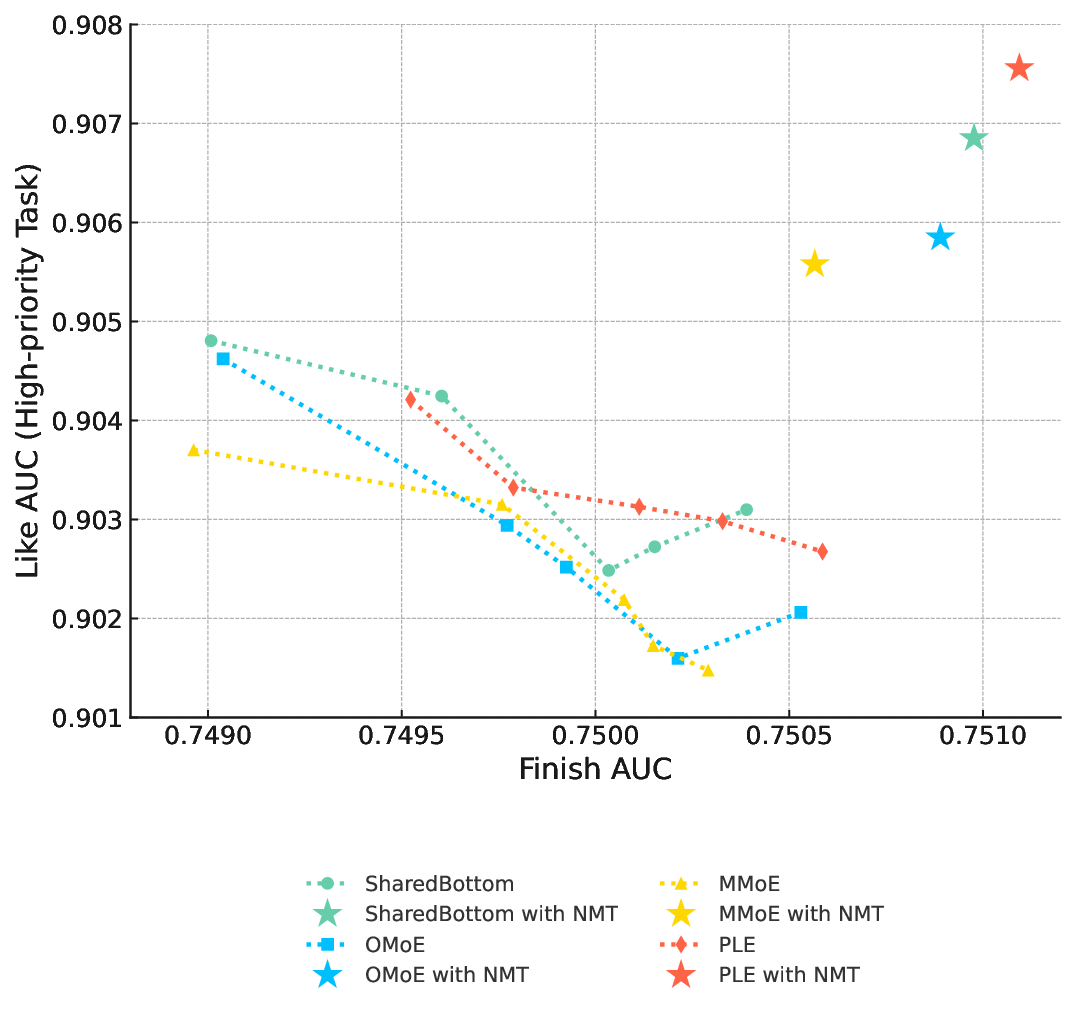}
    \caption{ AUC performance comparison for different model configurations across Like and Finish tasks. The colored lines are the performance of different models under different adjusted weights in loss function, and the star markers with the same color are the performance of respective NMT optimized model. The NMT-enhanced models demonstrate a significant improvement over their standard counterparts.}
    \label{fig:diffweight}
\end{figure}

\subsection{Online Experimental Results}
\subsubsection{Experimental Setup}
In our experiments, we aim to balance multiple business objectives within Taobao search, a large-scale online e-commerce search system that serves billions of active users. Our objectives are order volume, GMV (Gross Merchandise Value), and relevance, with order volume being the highest priority, followed by GMV, and relevance being the lowest priority. Order volume represents the number of successful transactions, GMV denotes the total monetary value of these transactions, and relevance assesses how well the search results align with user intent.

In contrast to public datasets commonly used in MTL research, where models typically output multiple scores to address different tasks, online ranking systems must provide a single score for final ranking. This requirement intensifies the conflict between tasks, as they must be integrated into a unified output score.

We utilize a standard Learning to Rank (LTR) paradigm with a deep neural network (DNN) architecture and the pairwise loss function \cite{exp:pairwise}. For the order volume task (\textit{pay}), we generate sample pairs by comparing positive samples (transactions with purchases) to negative samples (non-purchase interactions). For the GMV task (\textit{amount}), pairs are created by contrasting high-value transactions with lower-value transactions. For the relevance task (\textit{relevance}), pairs consist of relevant items versus non-relevant items, where relevant items are considered positive and non-relevant items are considered negative. The pairwise loss is formulated as:

\begin{equation}
    L = -\log\left(\text{sigmoid}(z_{pos} - z_{neg})\right)
\end{equation}

where $z$ is the output logit of the DNN. 

In our experiment, we consider \textit{pay} as the first priority task, \textit{amount} as the second priority task, and \textit{relevance} as the third priority task. We use a single-task model with only the \textit{pay} task as the baseline. 

\subsubsection{Performance Evaluation}

Table \ref{tab:online_result} shows the differences in business metrics between the multi-task ranking model and the baseline ranking model. 

\begin{table}[h]
    \centering
    \caption{Business Metrics of A/B Tests. Weight adjustment balances task weight manually. Experiments without annotations use the NMT framework to optimize multiple tasks. Baseline is a ranking model with only \textit{pay} task.}
    \label{tab:online_result}
    \resizebox{1.0\columnwidth}{!}{
        \begin{tabular}{@{}cccc@{}}
            \toprule
            Task & \makecell{Order\\Volume} & GMV & Relevance \\ \midrule
            \textit{pay} + \textit{relevance} & +0.26\% & +0.15\% & +0.72\% \\
            \makecell{\textit{pay} + \textit{relevance} \\ (weight adjustment)} & -0.35\% & -0.13\% & +0.42\% \\
            \textit{pay} + \textit{amount} & -0.05\% & +0.57\% & -0.24\% \\
            \textit{pay} + \textit{amount} + \textit{relevance} & -0.04\% & +0.49\% & +0.51\% \\ \bottomrule
        \end{tabular}
    }
\end{table}

We assessed the performance of integrating multiple tasks into a unified model using our proposed NMT framework. For the \textit{pay} + \textit{relevance} task setup, incorporating the relevance task resulted in a significant improvement in relevance metrics. Notably, the number of orders, reflecting the \textit{pay} task, did not decrease but showed a slight increase, suggesting that the prioritization of the \textit{pay} task was maintained while improving relevance.

In contrast, when we applied a weight adjustment method to balance task weights, we observed a clear enhancement in relevance metrics. However, this came at the cost of a notable decline in the number of orders. This decline can be attributed to the simple weighted summation of loss functions, which adversely impacted the primary \textit{pay} objective due to conflicting influences from the additional tasks.

The experiments with the \textit{pay} + \textit{amount} task demonstrated similar findings. We achieved a significant improvement in the \textit{amount} task while preserving the metrics for the \textit{pay} task. This confirms that our approach effectively maintains the performance of higher-priority tasks while enhancing secondary objectives.

To further validate the effectiveness of the NMT framework across multiple tasks, we extended our analysis to the \textit{pay} + \textit{amount} + \textit{relevance} setup. The results indicate that even when optimizing for all three tasks simultaneously, the NMT framework successfully preserved the performance of the highest-priority task (\textit{pay}) without any loss. Simultaneously, there was noticeable improvement in the metrics for the \textit{amount} and \textit{relevance} tasks.

\subsubsection{Optimization Process Visualization}

\begin{figure}[t]
    \centering
    \includegraphics[width=0.45\textwidth]{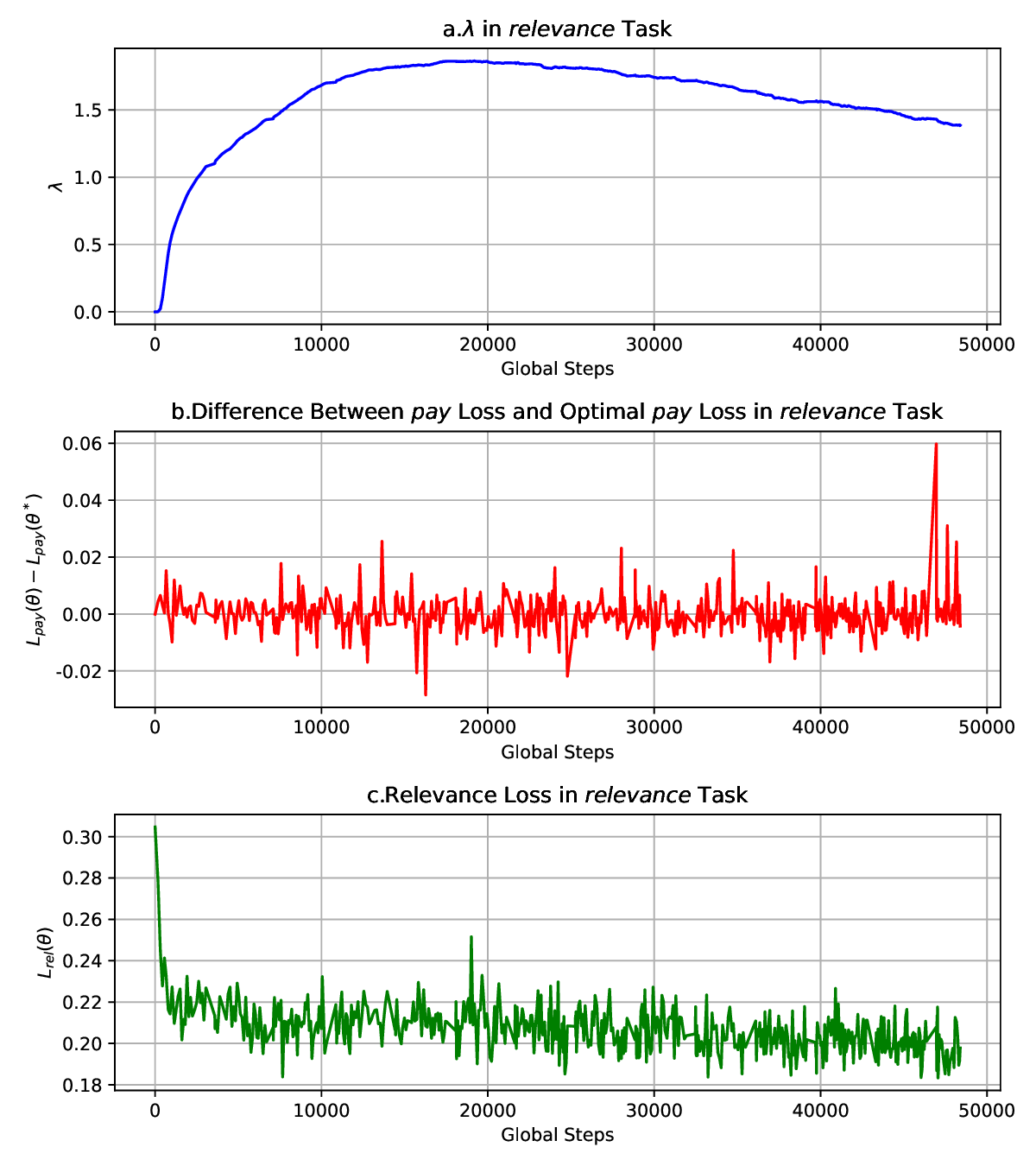}
    \caption{Training metrics of secondary task \textit{relevance} of \textit{pay} (High priority) + \textit{relevance} (Low priority). The top line shows the  $\lambda$ during training for the \textit{pay} + \textit{relevance} tasks. The middle line illustrates the fluctuation of the \textit{pay} task loss around its optimal value. The bottom line displays the loss function of the secondary task \textit{relevance} during the training process.}
    \label{fig:lambda_fig}
\end{figure}

To further illustrate the optimization process of the NMT framework, we analyze the \textit{pay} + \textit{relevance} task. In NMT framework, the optimization object of constrained \textit{relevance} task can be written as:

\begin{equation}
\setlength\arraycolsep{1.5pt}
\begin{array}{l@{\qquad}  c r c r}
    \displaystyle \min_{\theta} & \quad L(\theta) \\
         \text{where}  & \quad L(\theta)=L_{rel}(\theta) + \lambda(L_{pay}(\theta) - L_{pay}(\theta^*))  \\
\end{array}
\label{eq:constrain}
\end{equation}

The optimization process is depicted in Fig.\ref{fig:lambda_fig}. Initially, in Fig.\ref{fig:lambda_fig}.a, the constraint coefficient $\lambda$ is set to $0$, leading to a gradual increase in its value during the early stages of training. This increase indicates that while optimizing for the \textit{relevance} task alone, the \textit{pay} task struggles to maintain its optimal performance. However, as $\lambda$ rises, the \textit{pay} task loss stabilizes near its optimal value, with minor fluctuations. Throughout training, as Fig.\ref{fig:lambda_fig}.b shows, the \textit{pay} task loss remains close to its optimal point, demonstrating that the NMT framework effectively prioritizes the \textit{pay} task. Meanwhile, as we can see in Fig.\ref{fig:lambda_fig}.c, the loss of the \textit{relevance} task gradually decreases, indicating steady \textit{relevance} improvement as the training progresses. As $\lambda$ gradually converges, the model finds a balance between the \textit{pay} and \textit{relevance} tasks, ensuring that the \textit{pay} task achieves optimal performance while also enhancing the \textit{relevance} task.

Overall, the experiments underscore the efficacy of the NMT framework in balancing multiple objectives. By prioritizing the primary task while achieving substantial improvements in secondary tasks, the framework demonstrates robust performance in ranking scenarios in online industrial e-commerce systems, reflecting its practical utility in real-world applications.

\section{Conclusion}

No More Tuning (NMT) aims to solve a problem of multi-objective learning with different priorities that has long been ignored by the community. NMT effectively addressed the challenges of task prioritization and tedious hyper-parameter tuning. Through a constrained optimization approach, NMT ensures that primary tasks are optimized without compromising their performance while still refining secondary tasks. This method eliminates the need for manual parameter adjustments, provides theoretical analysis for task performance, and offers easy integration with existing gradient descent-based MTL methods. The broad applicability and proven effectiveness of NMT across various domains, from public datasets to large-scale industrial systems, highlight its versatility and potential for wide adoption.

\bibliography{aaai25}

\onecolumn
\begin{appendix}

\section{Appendix}
\subsection{Proof of Proposition 2}

To establish the convexity of the perturbation function \( P(\xi) \), we need to show that for any \(\xi_1, \xi_2 \in \mathbb{R}^m\) and \( t \in (0,1) \), the following inequality holds:

\begin{equation}
P(t * \xi_1 + (1-t)* \xi_2) \leq  t P(\xi_1) + (1-t) P(\xi_2)
\label{eqq1}
\end{equation}
\\
\noindent If either \(\xi_1\) or \(\xi_2\) does not have a feasible solution, then \(P(\xi_1)\) or \(P(\xi_2)\) would be positive infinity, making the inequality trivially satisfied.

\noindent For the remaining cases where perturbations keep the problem feasible, denote \(\theta^*(\xi_1)\) and \(\theta^*(\xi_2)\) as the optimal feasible parameters that satisfy \(P(\xi_1) = f(\theta^*(\xi_1))\) and \(P(\xi_2) = f(\theta^*(\xi_2))\), respectively.

\noindent Consider \(\theta_t = t \theta^*(\xi_1) + (1-t) \theta^*(\xi_2)\), which is a convex combination of \(\theta^*(\xi_1)\) and \(\theta^*(\xi_2)\).

\noindent Under Assumption2, where $f_i$ is Lipschitz continuous, we have: 
\begin{equation}
\|f_i(\theta_t) - f_i(\theta^*(\xi))\| \leq L \|\theta_t - \theta^*(\xi)\|,
\label{eqq2}
\end{equation}
where L is a constant.

\noindent Based on (\ref{eqq2}), we can analyze the following bounds:
\begin{equation}
\begin{aligned}
\|f_i(\theta_t) - f_i(\theta^*(\xi_1))\| & \leq L \|t \theta^*(\xi_1) + (1-t) \theta^*(\xi_2) - \theta^*(\xi_1)\| \\ & \leq L (1-t) \|\theta^*(\xi_2) - \theta^*(\xi_1)\|,
\end{aligned}
\end{equation}

\begin{equation}
\begin{aligned}
\|f_i(\theta_t) - f_i(\theta^*(\xi_2))\| &\leq L \|t \theta^*(\xi_1) + (1-t) \theta^*(\xi_2) - \theta^*(\xi_2)\| \\ &\leq L t \|\theta^*(\xi_2) - \theta^*(\xi_1)\|.
\end{aligned}
\end{equation}
Combining these bounds:
\begin{equation}
\begin{aligned}
\|f_i(\theta_t) - \left(t f_i(\theta^*(\xi_1)) + (1-t) f_i(\theta^*(\xi_2))\right)\| &\leq t \|f_i(\theta_t) - f_i(\theta^*(\xi_1))\| + (1-t) \|f_i(\theta_t) - f_i(\theta^*(\xi_2))\| \\
&\leq 2t(1-t)L \|\theta^*(\xi_2) - \theta^*(\xi_1)\|  \\
&\leq \frac{L}{2} \|\theta^*(\xi_2) - \theta^*(\xi_1)\|.
\end{aligned}
\end{equation}
We then write:
\begin{equation}
f_i(\theta_t) = t f_i(\theta^*(\xi_1)) + (1-t) f_i(\theta^*(\xi_2)) + \mathcal{O}(\epsilon),
\end{equation}
where $\epsilon = \|\theta^*(\xi_2) - \theta^*(\xi_1)\|.$

\noindent Since \( \epsilon \) is proportional to the distance \( \|\theta^*(\xi_2) - \theta^*(\xi_1)\| \), and given that \(\epsilon\) is small enough, \(\mathcal{O}(\epsilon)\) is negligible, and \(f_i\) behaves approximately linearly in the convex region.\\

\noindent As $\theta^*(\xi_1)$ and $\theta^*(\xi_2)$ are feasible solutions, they satisfy the constraints:  $f_i(\theta^*(\xi_1)) \leq f_i(\theta^*) + r_i - \xi_{1i}$ and $f_i(\theta^*(\xi_2)) \leq f_i(\theta^*) + r_i - \xi_{2i}$.

\noindent Thus:
\begin{equation}
\begin{aligned}
f_i(t \theta^*(\xi_1) + (1-t) \theta^*(\xi_2)) &\leq t f_i(\theta^*(\xi_1)) + (1-t) f_i(\theta^*(\xi_2)) \\
&\leq t (f_i(\theta^*) + r_i - \xi_{1i}) + (1-t)(f_i(\theta^*) + r_i - \xi_{2i})  \\
&\leq f_i(\theta^*) + r_i - (t\xi_{1} +(1-t)\xi_2)_i.
\end{aligned}
\end{equation}
Therefore, $ (1-t)\theta^*(\xi_1) + t \theta^*(\xi_2)$ is also a feasible parameter set that satisfies all the constraints of $P(t * \xi_1 + (1-t)* \xi_2)$. 

\noindent From the definition of the perturbation function, we have:
\begin{equation}
\begin{aligned}
P(t * \xi_1 + (1-t)* \xi_2)  &\leq f (t * \theta^*(\xi_1) + (1-t) * \theta^*(\xi_2)) \\ 
&\leq  tf (\theta^*(\xi_1))  + (1-t)f (\theta^*(\xi_2)) \\ &\leq tP(\xi_1) + (1-t)P(\xi_2)  
\end{aligned}
\label{eqq3}
\end{equation}

\noindent We have thus proven that \eqref{eqq1} holds, completing the proof.

\subsection{Time Complexity Analysis: From Exponential to Linear } 
Training efficiency is another key advantage of NMT. Existing MTL methods rely on grid search to determine the optimal task weights, balancing the prioritization of the primary task with the optimization of secondary tasks. For an $m$-task learning problem, the time complexity of grid search is $O(p^m)$, where $p$ represents the number of grid points. This is because grid search exhaustively evaluates all possible combinations of task weights, and for each task, $p$ grid points are tested. Consequently, the total number of combinations to evaluate grows exponentially with the number of tasks, resulting in $p^m$ evaluations.

In contrast, NMT only requires $m$ sequential steps to achieve the optimal solution, resulting in a time complexity of $O(m)$. This drastic reduction in time complexity, from exponential to linear, demonstrates the efficiency of NMT, making it highly scalable for problems with many tasks. 

\subsection{Experiments}

\subsubsection{Training Detail of Offline Experiments}
The learning rates for the model parameters \(\theta\) were varied across \(10^{-4}\), \(5 \times 10^{-4}\), and \(10^{-3}\), while the learning rates for the Lagrange multipliers \(\lambda\) were varied across \(10^{-2}\), \(5 \times 10^{-2}\), and \(10^{-1}\). Given the nine possible combinations, a grid search was conducted to identify the most suitable learning rate.
ReLU was used as the activation function for hidden layers. The hidden units for the expert and bottom MLPs were configured as [512, 512, 512], whereas the gate and tower MLPs had hidden units of [128, 64]. An L2 regularization of \(1 \times 10^{-6}\) was applied to the embeddings. The number of task-specific and shared experts was set to 1, and the batch size for training was 4096. All implementations were done in PyTorch, using the Adam optimizer\cite{exp:adam}. We adopted the same data pre-processing pipeline as the previous work in the same datasets \cite{mtl:STEM}. 

\subsubsection{Integration with Gradient-based Methods} 
In existing research and practice, parameter-sharing networks and gradient-based methods represent two major approaches to Multi-Task Learning (MTL). In the main paper, we explore the integration of the NMT framework with parameter-sharing networks, including Shared-Bottom \cite{mtl:Sharedbottom}, OMoE \cite{mtl:OMoE}, MMoE \cite{mtl:OMoE}, and PLE \cite{mtl:PLE}. This section focuses on the integration of NMT methods with gradient-based MTL approaches.

For simplicity, we first consider a dual-task MTL problem. In this context, the NMT framework optimizes the primary task \( f_1(\theta) \) in the first stage, followed by the optimization of the secondary task \( f_2(\theta) \) in the second stage, subject to the constraint that \( f_1(\theta^*) \), the optimal solution for the primary task, remains unchanged, as described in equation~\ref{eq:constrain_two_task}. The constrained stage of NMT can be viewed as a dual-task learning problem, consisting of the secondary task \( f_2(\theta) \) and the primary task constraint \( \lambda(f_2(\theta) - f_1(\theta^*)) \). From this perspective, any gradient-based method can be applied in the constrained stage of NMT.

We apply the state-of-the-art gradient-based MTL method, FAMO~\cite{famo}, within the constrained optimization framework of NMT. The experimental results are presented in Table~\ref{tab:famo_nmt}. For the primary task (i.e., "Like"), FAMO exhibits a substantial performance gap compared to NMT, primarily due to its lack of emphasis on the primary task. In the FAMO+NMT experiment, we observe an improvement in the AUC of the primary task, surpassing the performance of NMT alone. This demonstrates that NMT serves as a complementary framework, enhancing the performance of state-of-the-art MTL methods by prioritizing task importance rather than directly competing with them.

\begin{table*}[h]
\centering
\caption{Integration of gradient-based methods and NMT.}
\label{tab:famo_nmt}

\begin{tabular}{lcccccc}
\toprule
\textbf{}       & \multicolumn{2}{c}{\textbf{NMT}} & \multicolumn{2}{c}{\textbf{FAMO}} & \multicolumn{2}{c}{\textbf{FAMO+NMT}} \\ 
\cmidrule(lr){2-3} \cmidrule(lr){4-5} \cmidrule(lr){6-7}
\textbf{}       & \textbf{Finish AUC} & \textbf{Like AUC} & \textbf{Finish AUC} & \textbf{Like AUC} & \textbf{Finish AUC} & \textbf{Like AUC} \\ 
\midrule
\textbf{ShareBottom} & +0.06\%            & +0.38\%          & +0.07\%            & +0.21\%          & +0.04\%            & +0.41\%          \\ 
\textbf{OMoE}       & +0.04\%            & +0.38\%          & +0.05\%            & +0.11\%          & +0.07\%            & +0.43\%          \\ 
\textbf{MMoE}       & +0.03\%            & +0.41\%          & +0.05\%            & +0.14\%          & +0.04\%            & +0.40\%          \\ 
\textbf{PLE}        & +0.05\%            & +0.49\%          & +0.11\%            & +0.05\%          & +0.06\%            & +0.49\%          \\ 
\bottomrule
\end{tabular}

    \label{table1}
\end{table*}



\end{appendix}

\end{document}